\documentclass[11pt]{article}

\usepackage[top=1in, bottom=1in, left=0.9in, right=0.9in]{geometry}

\usepackage{amsmath}
\usepackage{epsfig}
\usepackage{amsthm}
\usepackage{thmtools}
\usepackage{thm-restate}
\usepackage{amssymb}
\usepackage{graphicx}
\usepackage{latexsym}
\usepackage{enumitem}
\usepackage{verbatim}

\usepackage{algorithm}
\usepackage{algorithmic}

\usepackage[square]{natbib}
\usepackage{nameref}
\usepackage{color}
\usepackage[colorlinks, linkcolor=blue, filecolor = blue, citecolor = blue, urlcolor = blue]{hyperref}

\newtheorem{theorem}{Theorem}
\newtheorem{lemma}[theorem]{Lemma}

\newtheorem{corollary}[theorem]{Corollary}
\newtheorem{definition}[theorem]{Definition}
\newtheorem{remark}[theorem]{Remark}

\renewcommand{\hat}{\widehat}
\renewcommand{\tilde}{\widetilde}

\renewcommand{\>}{{\rightarrow}}

\newcommand{\argmin}{\operatorname{argmin}}

\newcommand{\conv}{\operatorname{conv}}

\newcommand{\R}{{\mathbb R}}
\newcommand{\Z}{{\mathbb Z}}

\newcommand{\E}{{\mathbf E}}
\newcommand{\I}{{\mathbf I}}
\newcommand{\1}{{\mathbf 1}}

\newcommand{\cA}{{\mathcal A}}
\newcommand{\A}{{\mathbf A}}

\newcommand{\cH}{{\mathcal H}}

\newcommand{\cL}{{\mathcal L}}
\newcommand{\M}{{\mathbf M}}
\newcommand{\U}{{\mathbf U}}
\renewcommand{\H}{{\mathbf H}}

\newcommand{\X}{{\mathcal X}}
\newcommand{\Y}{{\mathcal Y}}

\renewcommand{\l}{{\boldsymbol l}}
\renewcommand{\a}{{\mathbf a}}

\newcommand{\e}{{\mathbf e}}

\newcommand{\p}{{\mathbf p}}
\newcommand{\q}{{\mathbf q}}

\renewcommand{\u}{{\mathbf u}}
\renewcommand{\v}{{\mathbf v}}
\newcommand{\w}{{\mathbf w}}
\newcommand{\x}{{\mathbf x}}
\newcommand{\y}{{\mathbf y}}
\newcommand{\z}{{\mathbf z}}
\newcommand{\floor[1]}{\lfloor#1\rfloor}

\newcommand{\OMD}{\textup{\textrm{{OMD}}}}

\begin{document}

\date{}

\title{\bfseries Online Learning for Structured Loss Spaces} 

\author{
Siddharth Barman\thanks{Indian Institute of Science. {\tt barman@csa.iisc.ernet.in} }, \and  Aditya Gopalan\thanks{Indian Institute of Science. {\tt aditya@ece.iisc.ernet.in}}, \and Aadirupa Saha\thanks{Indian Institute of Science. {\tt aadirupa.saha@csa.iisc.ernet.in}}
}

\maketitle

\begin{abstract}
 We consider prediction with expert advice when the loss vectors are assumed to lie in a set described by the sum of atomic norm balls. We derive a regret bound for a general version of the online mirror descent (OMD) algorithm that uses a combination of regularizers, each adapted to the constituent atomic norms. The general result recovers standard OMD regret bounds, and yields regret bounds for new structured settings where the loss vectors are (i) noisy versions of points from a low-rank subspace, (ii) sparse vectors corrupted with noise, and (iii) sparse perturbations of low-rank vectors. For the problem of online learning with structured losses, we also show lower bounds on regret in terms of rank and sparsity of the source set of the loss vectors, which implies lower bounds for the above additive loss settings as well. 

\end{abstract}

\section{Introduction}

%
%
%

Online learning problems, such as prediction with expert advice \citep{PLG06} and online convex optimization \citep{Zinkevich03}, involve a learner who sequentially makes decisions from a decision set. The learner seeks to minimize her total loss over a sequence of loss functions, unknown at the beginning, but which is revealed causally. Specifically, she attempts to achieve low regret, for each sequence in a class of loss sequences, with respect to the best single decision point in hindsight.

The theory of online learning, by now, has yielded flexible and elegant algorithmic techniques that enjoy provably sublinear regret in the time horizon of plays. Regret bounds for online learning algorithms typically hold across inputs (loss function sequences) that have little or no structure. For instance, for the prediction with experts problem, the exponentially weighted forecaster \citep{PLG06} is known to achieve an expected regret of $O(\sqrt{T \ln N})$ over any sequence of $N$-dimensional loss vectors with coordinates bounded in $[0,1]$; here $T$ is the number of rounds of play.

There is often, however, more geometric structure in the input in online learning problems, beyond elementary $\ell_\infty$-type constraints, which a learner, with a priori knowledge, can hope to exploit and improve her performance. A notable example is when the loss vectors for the prediction with experts problem come from a low-dimensional subspace \citep{HazanEtAl16}. This is often the case in recommender systems which are based on latent factor models \citep{KorBelVol09:MatFac}, where users and items are represented in terms of their features or attribute vectors, typically of small dimension. Under a bilinear model for the utility of a user-item pair, each user's utility across all items becomes a vector from a  subspace with dimension at most that of the feature vectors. \citet{HazanEtAl16} show that the learner can limit her regret to $O(\sqrt{dT})$ when each loss vector comes from a $d$-dimensional subspace of $\mathbb{R}^N$. If $d \ll N$ (in fact, $d \le \ln N$), then, this confers an advantage over using a more general best experts algorithm like Exponential Weights. 

This example is interesting not only because it shows that geometric/structural properties known in advance can help the learner achieve order-wise better regret, but also because it opens up the possibility of studying whether other, arguably more realistic, forms of structure can be exploited, such as sparsity in the input (or more generally small norm) and, more importantly, ``additive'' combinations of such structures, e.g., low-rank losses added with losses of small $\ell_2$-norm, which expresses losses that are noisy perturbations of a low-dimensional subspace. 
In this paper, we take a step in this direction and develop a framework for online learning problems with structured losses. 

\paragraph
{\bf Our Results and Techniques:} We consider the prediction with experts problem with loss sequences in which each element (loss vector) belongs to a set that respects structural constraints. Specifically, we assume that the loss vectors belong to a sum of atomic norm balls\footnote{centrally symmetric, convex, compact sets with their centroids at the origin.} \citep{Chandrasekaran+12}, say $A + B$, where the sum of sets is in the Minkowski sense.\footnote{$A + B = \{a + b: a \in A, b \in B\}$.} For this setup---which we call online learning with \emph{additive loss} spaces---we show a general regret guarantee for an online mirror descent (OMD) algorithm that uses a combination of regularizer functions, each of which is adapted to a constituent atomic norms of $A$ and $B$, respectively. 

Specializing this result for a variety of loss function sets recovers standard OMD regret guarantees for strongly convex regularizers \citep{Shw12:OLO}, and subsumes a result of \citet{HazanEtAl16} for the online low-rank problem. But more importantly, this allows us to obtain ``new results from old''---regret guarantees for settings such as noisy low rank (where losses are perturbations from a low-dimensional subspace), noisy sparse (where losses are perturbations of sparse vectors), and sparse low-rank (where losses are sparse perturbations from a low-dimensional subspace); see Tables~\ref{tab:omd_egs} and~\ref{tab:omd_add_egs}. 

Another contribution of this work is to show lower bounds on regret for the online learning problem with structured losses. We derive a generic lower bound on regret, for any algorithm for the prediction with experts problem, using structured (in terms of sparsity and dimension) loss vectors. This result allows us to derive regret lower bounds in a variety of individual and additive loss space settings including sparse, noisy, low rank, noisy low-rank, and noisy sparse losses. 

\textbf{Related work.} The work that is perhaps closest in spirit to ours is that of \citet{HazanEtAl16}, who study the best experts problem when the loss vectors all come from a low-dimensional subspace of the ambient space. A key result of theirs is that the online mirror descent (OMD) algorithm, used with a suitable regularization, improves the regret to depend only on the low rank and not the ambient dimension. More broadly, OMD theory provides regret bounds depending on properties of the regularizer and the geometry of the loss and decision spaces \citep{OCO-SS12}. In this work, we notably generalize this to the setup of additive losses.

Structured online learning has been studied in the recent past from the point of view of overall sequence complexity or ``hardness,'' also informally called learning with ``easy data.'' This includes work that shows algorithms enjoying first- and second-order regret bounds \citep{CBManSto07:SecondOrder} and quadratic variation-based regret bounds \citep{Hazan2010, pmlr-v32-steinhardtb14}. There is also recent work on achieving regret scaling with the covering number of the sequence of observed loss vectors \citep{CohenShie17}, which is another measure of easy data. 
 
 Our problem formulation, it can be argued, explores a different formulation of learning with ``easy data,'' in which the adversary, instead of being constrained to choose loss sequences with low total magnitude or variation, is limited to choosing loss vectors from sets with enough geometric structure (e.g., from particular atomic norm balls).

\section{Notation and Preliminaries}

For an integer $n \in \Z_{+}$, we use $[n]$ to denote the set $\{1,2, \ldots n\}$. For a vector $\x \in \R^n$, $x_i$ denotes the $i${th} component of $\x$. The $p$-norm of $\x$ is defined as $\|x\|_{p} = \left( \sum_{i = 1}^{n} |x_i|^p \right)^{1/p}$, $0 \le p < \infty$. Write $\|x\|_{\infty} := \max_{i = 1}^{n} |x_i|$ and $\| x \|_0 := | \{ i \mid x_i \neq 0 \}|$.
If $\|\cdotp\|$ is a norm defined on a closed convex set ${\Omega}\subseteq\R^n$, then its corresponding \emph{dual norm} is defined as 
\begin{align*}
\|\u\|^{*} 
	~ = ~
	\sup_{\x \in {\Omega}: \|\x\| \le 1} \x\cdot \u 
	\,,
\end{align*}
where $\x \cdot \u = \sum_{i} x_i u_i$ is the standard inner product in Euclidean space. It follows that the dual of the standard $p$-norm $(p \ge 1)$ is the $q$-norm, where $q$ is the H\"{o}lder conjugate of $p$, i.e., $\frac{1}{p} + \frac{1}{q} = 1$.
The \emph{$n$-probability simplex} is defined as $\Delta_{n} = \{\x \in [0,1]^{n} ~|~ \sum_{i = 1}^{n}x_i = 1\}$. 
Given any set $\cA \subseteq \R^n$, we denote the convex hull of $\cA$ as $\conv(\cA)$.
Clearly, when $\cA = \{\e_1, \e_2, \ldots \e_n\}$, $\conv{(\cA)} = \Delta_{n}$, where $\e_i \in [0,1]^{n}$ denotes $i${th} standard basis vector of $\R^{n}$.

\subsection{Atomic Norm and its Dual \citep{Chandrasekaran+12}}
\label{sec:at_norm}

Next we define the notion of an atomic norm along with its dual. These concepts will provide us with a unified framework for addressing structured loss spaces, and will be used extensively in the paper. Let $\cA \subseteq \R^n$ be a set which is convex, compact, and centrally symmetric about the origin (i.e., $\a \in \cA$ if and only if $-\a \in \cA)$. 

The atomic norm induced by the set $\cA$ is defined as 
\begin{align*}
||\x||_{\cA} := \inf \{t > 0 ~|~ \x \in t \cA \}, \quad \text{ for } \x \in \R^n.
\end{align*}
The dual of the atomic norm induced by $\cA$ becomes the \emph{support function} of $\cA$ \citep{Boy04:CO}; formally, 
\begin{align*}
||\x||^{*}_{\cA} := \sup \{\x.\z ~|~ \z \in \cA \}, \quad \text{ for } \x \in \mathbb R^n.
\end{align*}

For example, if the set $\cA$ is the convex hull of all unit-norm one-sparse vectors, i.e., $\cA := \conv\left( \{\pm \e_i\}_{i = 1}^{n} \right)$, then the corresponding atomic norm is the standard $\ell_1$-norm $\|\cdotp\|_{1}$.

\subsection{Problem setup} 
\label{sec:prob_setup}
We consider the online learning problem of learning with expert advice from a collection of $N$ experts \citep{PLG06}. In each round $t = 1, 2, \ldots, T$, the learner receives advice from each of the $N$ experts, following which the learner selects an expert from a distribution $\p_t \in \Delta_{N}$, maintained over the experts, whose advice is to be followed. %
%
%
 Upon this, the adversary reveals the losses incurred by the $N$ experts, $\l_t = (l_t(1), l_t(2), \ldots l_t(N)) \in [0,1]^{N}$, $l_t(i)$ being the loss incurred by the $i^{th}$ expert. The learner suffers an expected loss of $\E_{I_t \sim \p_t}[l_t({I_t})] = \sum_{i = 1}^{N}p_t(i)l_t({i})$. If the game is played for a total of $T$ rounds, then the objective of the learner is to minimize the expected cumulative regret defined as:
\begin{align*}
\E\left[\text{Regret}_{T}\right] = \sum_{t = 1}^{T}\p_t.\l_t - \min_{i \in [N]}\sum_{t = 1}^{T}l_t(i).
\end{align*}

It is well-known that without any further assumptions over the losses $\l_t$, the best achievable regret for this problem is $\Theta(\sqrt{T\ln N})$ -- the exponential weights algorithm or the Hedge algorithm achieves regret $O(\sqrt{T\ln N})$  \citep[Theorem $2.3$]{Arora+12}, and a matching lower bound exists as well \citep[Theorem $3.7$]{PLG06}.

 Now, a very natural question to ask is: can a better (smaller) regret be achieved if the loss sequence has more structure? Suppose the loss vectors $(\l_t)_{t = 1}^{T}$ all belong to a common {\em structured loss space} $\cL \subseteq [0,1]^N$, such as:
\begin{enumerate}
 \item Sparse loss space: $\cL = \{\l \in [0,1]^{N} ~|~ \|\l\|_{0} = s\}$. Here, $s \in [N]$ is the sparsity parameter.
 \item Spherical loss space: $\cL = \{\l \in [0,1]^{N} ~|~ \|\l\|_{\A} = \l^{\top}\A\l \le \epsilon\}$, where $\A$ is a positive definite matrix and $\epsilon > 0$.
 \item Noisy loss space: $\cL = \{\l \in [0,1]^{N} ~|~ \|\l\|_{2}^2 = \epsilon\}, ~\epsilon>0\}$. Note that \emph{noisy losses} are a special class of \emph{spherical losses} where $\A = \I_{N}$, the identity matrix.
 \item Low-rank loss space: $\cL = \{\l \in [0,1]^{N} ~|~  \l = \U \v ~\}$, where the rank of matrix $\U \in \R^{N \times d}$ is equal to $d \in [N]$ and vector $\v \in \R^{d}$ (as mentioned previously, such loss vectors were considered by \cite{HazanEtAl16}).
\item Additive loss space: $\cL = \cL_1 + \cL_2$ (Minkowski Sum). More formally, $\cL = \{\l = \l_{1} + \l_{2} \mid  \l_1 \in \cL_1 \text{ and } \l_2 \in \cL_2 \}$, where $\cL_1 \subseteq [0,1]^{N}$ and $\cL_2 \subseteq [0,1]^{N}$ are structured loss spaces themselves.\footnote{Note that, in the problem setup at hand the learner observes only the loss vectors $l_t$, and does not have access to the loss components $l_{1t}$ or $l_{2t}$.} Examples include any combination of the previously described loss spaces, such as the  low-rank + noisy space. 
\end{enumerate}

Clearly, using the Exponential Weight or Hedge algorithm, one can always achieve $O(\sqrt{T \ln N})$ regret in the above settings. The relevant question is whether the geometry of such loss spaces can be exploited, in a principled fashion, to achieve improved regret guarantees (possibly independent of $\ln N$)? In other words, can we come up with algorithms for above cases such that the regret is $O(\sqrt{\omega T})$, where $\omega < \ln N$?

We will show that, for all of the above loss spaces, we can obtain a regret factor $\omega$ which is order-wise better than $\ln N$. In particular, we will establish these regret bounds by employing the Online Mirror Descent algorithm (described below) with a right choice of atomic norms. Furthermore,  using this algorithm, we will also develop a framework to obtain new regret bounds from old. That is, we show that if we have an online mirror descent setup for $\cL_1$ and $\cL_2$, then we can in fact obtain a low-regret algorithm for the additive loss space $\cL_1 + \cL_2$.

\subsection{Online Mirror Descent}
\label{sec:omd}
In this section, we give a brief introduction to the Online Mirror Descent (OMD) algorithm \citep{Bubeck11,OCO-SS12}, which is a subgradient descent based method for online convex optimization with a suitably chosen regularizer. A reader well-versed with the analysis of OMD may skip to the statement of Theorem \ref{thm:omd} and proceed to Section \ref{sec:omd_egs}. 

OMD generalizes the basic mirror descent algorithm used for offline optimization problems (see, e.g.,~\cite{BeckTe03}). 
Before detailing the algorithm, we will recall a few relevant definitions: 

\begin{definition}
\textbf{Bregman Divergence.}  
Let ${\Omega}\in\R^n$ be a convex set, and $f \colon {\Omega}\> \R$ be a strictly convex and differentiable function. Then the \emph{Bregman divergence} associated with $f$, denoted by $B_f:{\Omega}\times{\Omega}\>\R$, is defined as 
\begin{align*}
B_{f}(\u,\v) := f(\u) -f(\v) -(\u-\v)\cdot\nabla f(\v), ~~\qquad \text{ for } \u,\v\in {\Omega}
	\,.
\end{align*}
\end{definition}

\begin{definition}{\textbf{Strong Convexity (see, e.g., \cite{OCO-SS12} and \cite{Bubeck11})}}
Let ${\Omega}\in\R^n$ be a convex set, and $f:{\Omega}\> \R$ be a differentiable function. Then $f$ is called $\alpha$-strongly convex over ${\Omega}$ with respect to the norm $\|\cdot\|$ iff for all $\x, \y \in {\Omega}$,
\begin{align*}
 ~~ f(\x) - f(\y) - (\nabla f(\y))^{T}(\x - \y) \ge \frac{\alpha}{2} \|\x - \y\|^{2}. 
\end{align*}
Equivalently, a continuous twice differentiable function, $f$, over $\Omega$ is said to be $\alpha$-strongly convex iff for all $\x, \w \in \Omega$ we have 
\begin{align*}
 \x^{T}\nabla^{2}f(\w)\x \ge \alpha \|\x\|^{2}. 
\end{align*}
\end{definition}
We now describe the OMD algorithm instantiated to the problem setup given in Section \ref{sec:prob_setup}. 

\begin{center}
\begin{algorithm}[h]
   \caption{\textbf{Online Mirror Descent (OMD)}}
   \label{alg:B-MultiSBM}
\begin{algorithmic}[1]
   \STATE {\bfseries Parameters:} Learning rate $\eta > 0$.
   \STATE ~~~~~~~~~~~~~~~~~~~~~~ Convex set ${\Omega} \subseteq \R^{N}$, such that $\Delta_{N} \subseteq {\Omega}$
   \STATE ~~~~~~~~~~~~~~~~~~~~~~ Strictly convex, differentiable function $R: {\Omega} \> \R$
   \STATE {\bfseries Initialize:} $\p_{1} = \underset{\p \in \Delta_{N}}{\text{argmin}} \, R(\p)$
   \FOR{$t = 1,2, \cdots T$}
   \STATE Play $\p_{t} \in \Delta_{N}$
   \STATE Receive loss vector $\l_{t} \in [0,1]^{N}$ 
   \STATE Incur loss $\p_{t}.\l_{t}$ 
   \STATE Update:
   \STATE ~~~$\nabla R(\tilde{\p}_{t+1}) \leftarrow  \nabla R({\p}_{t}) - \eta \l_{t}$ 
	\quad(Assume this yields $\tilde{\p}_{t+1} \in {\Omega}$)
    \STATE ~~~${\p}_{t+1} \leftarrow \underset{\p \in \Delta_{N}}{\text{argmin}} \, B_{R}(\p,\tilde{\p}_{t+1})$
   \ENDFOR
\end{algorithmic}
\end{algorithm}
\vspace{2pt}
\end{center}

The regret guarantee of the above algorithm is as follows: 

\begin{restatable}[\textbf{OMD regret bound (Theorem $5.2$, \cite{Bubeck11})}]{theorem}{omd}
\label{thm:omd}
Let the loss vectors, $\{\l_{t} \}_{t = 1}^{T}$, belong to a loss space $\cL \subseteq [0,1]^{N}$, which is bounded with respect to a (arbitrary) norm $\|\cdotp\|$; in particular, for any $\l \in \cL$ we have $\| \l \| \le G$.  
Furthermore, let $\Omega \supseteq \Delta_{N}$ be a convex set, and $R: {\Omega} \> \R$ be a  strictly convex, differentiable function that satisfies $R(\p) - R(\p_1) \le ~ D^{2}$ for parameter $D \in \R$ and all $\p \in \Delta_N$; where 
$
\p_1 := \argmin_{p \in \Delta_N}  R(\p)
$. Also, let the restriction of $R$ to $\Delta_{N}$ be $\alpha$-strongly convex with respect to $\|\cdotp\|^{*}$, the dual norm of $\|\cdotp\|$.

Then, the regret of OMD algorithm with set ${\Omega}$, regularizer function $R$, and learning rate $\eta > 0$, for $T$ rounds satisfies 
\begin{align*}
\text{Regret}_{T}(\OMD(\eta^{*})) = \sum_{t = 1}^{T}\p_{t}.\l_{t} - \min_{i = 1}^{N} \sum_{t = 1}^{T}l_{t}(i) 
	~ \le ~
	\frac{1}{\eta}\left( D^{2} + \frac{\eta^{2}G^{2}T}{2\alpha} \right) 
	\,,
\end{align*}
where $\p_1, \p_2, \ldots \p_T$ denotes the sequential predictions of the algorithm in $T$ rounds. Moreover, setting $\eta^{*} = \frac{D}{G}\sqrt{\frac{2\alpha}{T}}$ (i.e., minimizing the right-hand-side of the above bound), 
we have 
\begin{align*}
\text{Regret}_{T}(\OMD(\eta^{*}))
	~ \le ~ 
	DG\sqrt{\frac{2T}{\alpha}} 
	\,.
\end{align*}
\end{restatable}
For completeness, a proof of the above theorem appears in Appendix \ref{app:omd}.

\section{Online Mirror Descent for Structured Losses}
\label{sec:omd_struct}
\label{sec:omd_egs}

This section shows that, for specific structured loss spaces, instantiating the OMD algorithm---with a right choice of the norm $\| \cdot \|$ and regularizer $R$---leads to improved (over the standard $O(\sqrt{2 T \ln N})$ bound) regret guarantees. The proofs of these results appear in Appendix \ref{app:omd_prfs}.

\begin{table}[h]
\begin{center}
\begingroup
\setlength{\tabcolsep}{3pt} 
\renewcommand{\arraystretch}{1.5} 
\begin{tabular}{|c|c|c|c|}
\hline
\textbf{Loss Space}  & \textbf{Regret Bound} & \textbf{Atomic Norm} & \textbf{Regularizer} \\
\hline
$s$-Sparse & {$2\sqrt{{\ln (s+1)T}}$} & {$ \frac{1}{\sqrt{2}}\|\cdot\|_{p}$} & {$ \|\x \|_{ q}^2$} \\
{} & {} & {$(p = 2\ln (s+1))$} & {$(q = \frac{p}{p-1})$}  \\
\hline
Spherical & {$\sqrt{\lambda_{\max}(\A^{-1}){T}}$} & {$ \|\cdot\|_{A}$} & {$ \x^{\top}\A^{-1}\x$} \\
\hline
$\epsilon$-Noise & {$\sqrt{{\epsilon T}}$} & {$\frac{1}{\sqrt{\epsilon}}\|\cdot\|_{2}$} & {$\epsilon\x^{\top}\x$}  \\
\hline
\end{tabular}
\endgroup
\caption{OMD Regret Bounds for Structured Loss Spaces}
\label{tab:omd_egs}
\end{center}
\end{table}

\begin{enumerate}
\item {Sparse loss space:} $\cL = \{\l \in [0,1]^{N} ~|~ \|\l\|_{0} = s\}$, $s \in [N]$ being the loss sparsity parameter. Then using $q$-norm, $R(\x) = \|\x \|_{ q}^2 = \left(\sum_{i = 1}^{N}(x_i^{q})\right)^{\frac{2}{q}}$, where $q = \frac{\ln s'}{\ln s' - 1}$, $s' = (s+1)^2$, as the regularizer, we get,
\begin{align*}
\text{Regret}_{T} 
	~ \le ~ 
	2\sqrt{{\ln (s+1)T}} 
	\,.
\end{align*}

\item {Spherical loss space:} $\cL = \{\l \in [0,1]^{N} ~|~ \|\l\|_{\A}^2 = \l^{\top}\A\l \le \epsilon\}$, where $\A$ is a positive definite matrix, $\epsilon > 0$. Then using the square of the ellipsoidal norm as the regularizer, $R(\x) = {\epsilon}\x^{\top}\A^{-1}\x$, we get,
\begin{align*}
\text{Regret}_{T}
	~ \le ~ 
	\sqrt{\lambda_{\max}(\A^{-1}){\epsilon T}} 
	\,,
\end{align*}
where $\lambda_{\max}(\A^{-1})$ denotes the maximum eigenvalue of $\A^{-1}$.

\item {Noisy loss sapce:} $\cL = \{\l \in [0,1]^{N} ~|~ \|\l\|_{2}^{2} \le \epsilon\}, ~\epsilon > 0$. Then using the square of the standard Euclidean norm  as the regularizer, $R(\x) = \epsilon\|\x\|_{2}^2$, we get,
\begin{align*}
\text{Regret}_{T}
	~ \le ~ 
	\sqrt{{\epsilon T}} 
	\,.
\end{align*}
Note that \emph{noisy loss} is a special case of \emph{spherical loss} where $\A = \A^{-1} = \I_{N}$. \\
\end{enumerate}

\cite{HazanEtAl16} have also used OMD to address the loss vectors that belong to a low-dimensional subspace. Specifically, if the loss space $\cL = \{\l \in [0,1]^{N} ~|~  \l = \U \v ~\}$, where matrix $\U \in \R^{N \times d}$ is of rank $d$ and vector $\v \in \R^{d}$, with $1 \le d \le \ln N$. Then, \cite{HazanEtAl16} show that the regularizer $R(\x) = \|\x\|_{\H}^{2} = \x^{\top} \H \x$ (where $\H = \I_{N} + \U^{\top} \M \U $, $\M$ is the matrix corresponding to the L\"owner-John ellipsoid of $\cL$ and $\I_{N} $ is the identity matrix) leads to the following regret bound:
\begin{align*}
\text{Regret}_{T}
	~ \le ~ 
	4\sqrt{{dT}} 
	\,.
\end{align*} 

In addition, for the standard loss space $\cL = [0,1]^{n}$, one can execute the OMD algorithm with the unnormalized negative entropy, $R(\x) = \sum_{i = 1}^{N}x_i \log x_i - \sum_{i = 1}^{N} x_i$, as the regularizer, to obtain: 
\begin{align*}
\text{Regret}_{T}
	~ \le ~ 
	\sqrt{{2T \ln N}} 
	\,.
\end{align*}

Note that the above regret bound is same as that given by Hedge algorithm. In fact, it can be verified that, with the above choice of regularizer, the OMD algorithm exactly reduces down to standard Hedge algorithm (see, e.g., \cite{Bubeck11}).

\section{Online Learning for Additive Loss Spaces}
\label{sec:omd_add}
We now present a key result of this paper, which enables us to obtain new regret bounds from old. In particular, we will develop a framework that provides a low-regret OMD algorithm for an additive loss space $\cL = \cL_1 + \cL_2$, using the OMD setup of the constituent loss spaces $\cL_1$ and $\cL_2$.  Specifically, we detail how to choose an appropriate regularizer for losses from $\cL$ and, hence, construct a low-regret OMD algorithm. 


\begin{theorem} \textbf{(Main Result)} 
\label{thm:omd_add}
Let $\cL_1, \cL_2 \subseteq [0,1]^{N}$ be two loss spaces, such that $\cL_1 \subseteq \cA_1$, $\cL_2 \subseteq \cA_2$, where $\cA_1, ~\cA_2 \in \R^{N}$ are two centrally symmetric, convex, compact sets. We observe a sequence of loss vectors $\{\l_t\}_{t=1}^{T}$, such that in any round $t \in [T]$, $\l_t = \l_{1t} + \l_{2t}$, where $\l_{1t} \in \cL_1$ and $\l_{2t} \in \cL_2$. Consider two differentiable, strictly convex functions $R_1: \Omega_1 \mapsto \R$, $R_2: \Omega_2 \mapsto \R$, where $\Omega_1, \Omega_2 \supseteq \Delta_{N}$ are two convex sets. The restrictions of $R_1$ and $R_2$ to $\Delta_{N}$ are, respectively, $\alpha_1$- and $\alpha_2$-strongly convex with respect to the norms $||\cdot||^{*}_{\cA_1}$ and $||\cdot||^{*}_{\cA_2}$. 

Also, let parameters $D_1$ and $D_2$ be such that  
 $R_1(\p) - R_1(\p_1) \le D_1^{2}$ and 
$R_2(\p) - R_2(\p_1) \le D_2^{2}$ for all $ \p \in \Delta_N$;
where $\p_1 := \argmin_{\p \in \Delta_N} \left( R_1(\p) + R_2(\p) \right)$.

Then (with learning rate $\eta^* = \sqrt{\frac{(D_1^2 + D_2^2) \min(\alpha_1, \alpha_2) }{T}}$, regularizer $R := R_1 + R_2$, and $\p_1$ as the initial prediction) the regret of the OMD algorithm is bounded as
\begin{align*}
\text{Regret}_{T} 
	~ \le ~ 
	 2\sqrt{\frac{\left(D_1^2 + D_2^2 \right) T}{\min(\alpha_1, \alpha_2)}} 
	\,.
\end{align*}


\end{theorem} 
\hfill \\

A proof of the above theorem appears in Section~\ref{sec:omd_add_proof}. \\

\begin{remark}
In general, the regret guarantee of Theorem~\ref{thm:omd_add} is essentially tight. That is, there exist loss spaces $\cL_1$ and $\cL_2$ such that OMD algorithm obtained via Theorem~\ref{thm:omd_add} provides an order-wise optimal regret bound for the additive loss space $\cL = \cL_1 + \cL_2$; see Appendix~\ref{appendix:tight-examples} for specific examples. 
\end{remark}

The above theorem immediately leads  to the following corollary.

\begin{corollary} \textbf{(New Regret Bounds from Old)}  
\label{corr:omd_add} Suppose $\cL_1, \cL_2 \subseteq [0,1]^{N}$ are two loss spaces such that $\|\l\|_{\cA_1} \le 1, ~\forall \l \in \cL_1$, and $\|\l\|_{\cA_2} \le 1, ~\forall \l \in \cL_2$, where $\cA_1,\cA_2 \in \R^{N}$ are two centrally symmetric, convex, compact sets.
Also, suppose there exists two strictly convex, differentiable functions $R_1: \Omega_1 \mapsto \R$
and $R_2: \Omega_2 \mapsto \R$, ($\Omega_1, \Omega_2 \supseteq \Delta_{N}$, convex) such that OMD with regularizer functions $R_1$ and $R_2$ gives the regret bounds of $D_1\sqrt{\frac{2T}{\alpha_1}}$ and $D_2\sqrt{\frac{2T}{\alpha_2}}$ over loss spaces $\cL_1$ and $\cL_2$, respectively. 
Here, $\alpha_1$ $(\alpha_2)$ is the strong convexity parameter of $R_1$ $(R_2)$ over $\Delta_N$, with respect to the atomic norm $||\cdot||^{*}_{\cA_1}$ $(||\cdot||^{*}_{\cA_2})$. 

In addition let, $D_1$ and $D_2$ are parameters such that, for all $\p \in \Delta_N$, 

\begin{align*}
R_1(\p) - R_1(\p_1')   & \leq  D_1^{2}  \qquad \text{ with }  \p_1' := \argmin_{\q \in \Delta_N} \ R_1(\q) \text{ and} \\
R_2(\p) - R_2(\p_2')  & \leq D_2^{2} \qquad \text{ with }  \p_2' = \argmin_{\q \in \Delta_N} \ R_2(\q).
\end{align*}

Then, for the additive loss space $\cL = \cL_1 + \cL_2$, the OMD algorithm with regularizer function $R = R_1 + R_2$, initial prediction $\p_1 = \argmin_{\p \in \Delta_N} \left( R_1(\p) + R_2(\p) \right)$ (and learning rate $\eta^{*} = \sqrt{\frac{(D_1^2 + D_2^2)\min(\alpha_1, \alpha_2)}{T}}$) leads to the following regret bound:
\begin{align*}
\text{Regret}_{T}
	~ \le ~ 
	2\sqrt{\frac{(D_1^2 + D_2^2)T}{\min(\alpha_1, \alpha_2)}} 
	\,.
\end{align*}
\end{corollary}

\hfill \\

Note that we can prove this corollary---using Theorem~\ref{thm:omd_add}---by simlply verifying the following inequalities: $R_1(\p) - R_1(\p_1) \le D_1^{2}$ and $R_2(\p) - R_2(\p_1) \le D_2^{2}$, for all $\p \in \Delta_N$ and $ \p_1 := \argmin_{\q \in \Delta_N} \left( R_1(\q) + R_2(\q) \right)$. This follows, since $R_1(\p_1') \leq R_1(\p_1)$ and $R_2(\p_2') \leq R_2(\p_1)$; recall that $\p_1':= \argmin_{\q \in \Delta_N} R_1(\q)$ and $\p_2' := \argmin_{\q \in \Delta_N} R_2(\q)$.

\subsection{Proof of Theorem~\ref{thm:omd_add}}
\label{sec:omd_add_proof}

Before proceeding to prove the theorem, we will establish the following useful lemmas. Let $\cA_1, \cA_2$ be any two convex, compact, centrally symmetric subsets of $\R^{n}$ and $\cA = \cA_1 + \cA_2$ (Minkowski Sum). 
Then, note that $\cA$ is also convex, compact, and centrally symmetric. This follows from the fact that $\conv(\mathcal{X}) + \conv(\mathcal{Y}) = \conv(\mathcal{X+Y})$ for any $\mathcal{X}, \mathcal{Y} \subset \R^n$. In addition, we have

\begin{lemma} 
\label{lem:norm} $||\x||_{\cA} \le \max\{ ||\x_1||_{\cA_1}, ||\x_2||_{\cA_2} \}$, where $\x = \x_1 + \x_2$, $\x_1 \in \cA_1$, $\x_2 \in \cA_2$. 
\end{lemma}

\begin{proof}
Recall the definition of atomic norm $\|\cdot\|_{\cA}$ from Section \ref{sec:at_norm}. Suppose for any $\x = (\x_1 + \x_2) \in \R^{n}$, $t_1 = ||\x_1||_{\cA_1}$ and $t_2 = ||\x_2||_{\cA_2}$. Clearly, $\x = \x_1 + \x_2 \in (t_1 \cA_1 + t_2 \cA_2) \subseteq t(  \cA_1 + \cA_2 )$, where $t = \max\{t_1,t_2\}$. The proof now follows directly from the definition of atomic norm, $\|\x\|_{\cA}$.
\end{proof}


\begin{lemma} 
\label{lem:dual-norm} 
$||\x||^{*}_{\cA} \le  ||\x||^{*}_{\cA_1} + ||\x||^{*}_{\cA_2}$, for all $ \x \in \R^n$. 
\end{lemma}

\begin{proof}
Consider any $\x \in \R^n$, 
\begin{align*}
||\x||^{*}_{\cA} 
	& = 
	\sup \{\x.\z ~|~ \z \in \cA \}
\\
	& =  
	\sup \{\x.(\z_1+\z_2) ~|~ \z_1 \in \cA_1, ~ \z_2 \in \cA_2 \}
\\
	& \le  
	\sup \{\x.\z_1 ~|~ \z_1 \in \cA_1 \} + \sup \{\x.\z_2 ~|~ \z_2 \in \cA_2\}
\\
	& \le 
	||\x||^{*}_{\cA_1} + ||\x||^{*}_{\cA_2}
	\,.
\end{align*} 
\end{proof}

\begin{lemma} 
\label{lem:reg} 
Suppose $\tilde{\Omega} \in \R^n$ be a convex set. Consider two differentiable functions $R_1: \R^n \mapsto \R$ and $R_2: \R^n \mapsto \R$, that are respectively $\alpha_1$ and $\alpha_2$-strongly convex with respect to $||\cdot||^{*}_{\cA_1}$ and $||\cdot||^{*}_{\cA_2}$ over $\tilde{\Omega}$. Then $R = R_1 + R_2$ is $\alpha = \frac{1}{2}\min(\alpha_1, \alpha_2)$-strongly convex with respect to $||\cdot||^{*}_{\cA}$ over $\tilde{\Omega}$. 
\end{lemma}

\begin{proof}
For any $\x,\y \in \tilde{\Omega}$,
\begin{align*}
R(\x) - R(\y) - & \nabla R(\y)(\y)(\x - \y)
\\    
	& = 
	R_1(\x) - R_1(\y) - \nabla R_1(\y)(\y)(\x - \y)
	+
	R_2(\x) - R_2(\y) - \nabla R_2(\y)(\x - \y)
\\
	& =  
	\frac{\alpha_1}{2}\|\x-\y\|_{\cA_1}^{*2}
	+ 
	\frac{\alpha_2}{2}\|\x-\y\|_{\cA_2}^{*2}
\\
	& \ge  
	\frac{\alpha}{2} (2\|\x - \y\|^{*2}_{\cA_1} + 2\|\x\|^{*2}_{\cA_2}), ~~(\alpha = \frac{1}{2}\min(\alpha_1, \alpha_2))
\\
	& \ge  
	\frac{\alpha}{2} (\|\x - \y\|^{*}_{\cA_1} + \|\x - \y\|^{*}_{\cA_2})^{2} \,~~(\text{since } 2(a^2 + b^2) > (a+b)^2, ~ \forall a,b \in \R)
\\
	& \ge 
	\frac{\alpha}{2} (\|\x - \y\|^{*2}_{\cA}) \,~~(\text{via Lemma } \ref{lem:dual-norm})
	\,.
\end{align*} 
Hence, $R = R_1 + R_2$ is $\alpha = \frac{1}{2}\min(\alpha_1, \alpha_2)$-strongly convex with respect to $||\cdot||^{*}_{\cA}$ over $\tilde{\Omega}$. Similarly, if $R_1$ and $R_2$ are twice continuously differentiable, then for any $\x,\w \in \tilde{\Omega}$.

\begin{align*}
\x^{T}\nabla^{2}R(\w)\x     
	& = 
	\x^{T}\nabla^{2}(R_1+R_2)(\w)\x 
\\
	& = 
	\x^{T}\nabla^{2}R_1(\w)\x + \x^{T}\nabla^{2}R_2(\w)\x 
\\
	& \ge  
	\alpha_1 \|\x\|^{*2}_{\cA_1}  + \alpha_2 \|\x\|^{*2}_{\cA_2}
\\
	& \ge  
	\alpha (2\|\x\|^{*2}_{\cA_1} + 2\|\x\|^{*2}_{\cA_2}), ~~(\alpha = \frac{1}{2}\min(\alpha_1, \alpha_2))\\
	& \ge  
	\alpha (\|\x\|^{*}_{\cA_1} + \|\x\|^{*}_{\cA_2})^{2} \,~~(\text{Since } 2(a^2 + b^2) > (a+b)^2, ~ \forall a,b \in \R)
\\
	& \ge 
	\alpha (\|\x\|^{*}_{\cA})^2 \,~~(\text{via Lemma } \ref{lem:dual-norm})
	\,.
\end{align*} 
Thus $R$ is $\frac{1}{2}\min(\alpha_1, \alpha_2)$-strongly convex with respect to $||\cdot||^{*}_{\cA}$ over $\tilde{\Omega}$.
\end{proof}

\begin{proof}{\textbf{of Theorem \ref{thm:omd_add}} \ \ }
Consider the norm $\|\cdot\| = \|\cdot\|_{\cA}$, and its dual norm $\|\cdot\|^{*} = \|\cdot\|_{\cA}^{*}$. Note that: 
\begin{enumerate}
\item Lemma~\ref{lem:norm} along with the bounds $\|\l_1\|_{\cA_1} \le 1$ and $\|\l_2\|_{\cA_2} \le 1$ imply that  $\|\l\|_{\cA} \le 1$, for any $\l  = \l_1 + \l_1 \in \cL$. Hence, $\cL \subseteq \cA$.
\item For any $\p \in \Delta_{N}$, $R(\p) - R(\p_1)  = (R_1(\p) - R_1(\p_1)) + (R_2(\p) - R_2(\p_1)) \le D_1^2 + D_2^2$. Hence, $D = \sqrt{D_1^2 + D_2^2}$.
\item $R(\x) = R_1(\x) + R_2(\x)$ is $\frac{\min\{\alpha_1, \alpha_2\}}{2}$-strongly convex with respect to $\|\cdot\|_{\cA}^{*}$, $\forall \x \in \Delta_{N}$ (Lemma \ref{lem:reg}). Hence, $\alpha = \frac{\min\{\alpha_1, \alpha_2\}}{2}$. 
\end{enumerate}
The result now follows by applying Theorem \ref{thm:omd}. 
\end{proof}

\vspace{-20pt}
\subsection{Applications of Theorem~\ref{thm:omd_add}}
\label{sec:omd_add_egs}

In this section, we will derive novel regret bounds for additive loss spaces  ($\cL = \cL_1 + \cL_2$) wherein the individual components ($\cL_1$ and $\cL_2$) are the loss spaces which were considered in Section~\ref{sec:omd_egs}. These results are derived by applying Theorem \ref{thm:omd_add}; details of the proofs appear in Appendix~\ref{app:omd_add_egs}.

\begin{table}[h]
\begin{center}
\begingroup
\setlength{\tabcolsep}{5pt} 
\renewcommand{\arraystretch}{1.6} 
\begin{tabular}{|c|c|c|c|}
\hline
\textbf{Loss Space}  & \textbf{Regret Bound} & \textbf{Atomic Norm} & \textbf{Regularizer} \\
\hline
$d$-Low Rank & {} & {$\|\cdot\|_{\cA},~ \cA = \cA_1 + \cA_2,$ where }  & {} \\
{+ $\epsilon$-Noise } & {$\sqrt{2(16d+\epsilon)T}$} & {$\cA_1 =  \left\lbrace \x \in \R^{N} ~|~ \sqrt{\x^{\top}{\H}^{-1}\x} \le 1 \right\rbrace$,} & {$\|\x\|_{\H}^2 + {\epsilon}\|\x\|_{2}^2$} \\
{} & {} & {$\cA_2 = \left\lbrace \x \in \R^{N} ~|~ \frac{1}{\sqrt{\epsilon}}\sqrt{\x^{T}\x} \le 1 \right\rbrace$.} & {} \\
\hline
{$s$-Sparse}  & {} & {$\|\cdot\|_{\cA},~ \cA = \cA_1 + \cA_2,$ where}  & {} \\
{+ $\epsilon$-Noise} & {$2\sqrt{2(1+\epsilon)\ln (s+1)T}$} & {$\cA_1 = \left\lbrace \x \in \R^{N} ~|~ \frac{1}{\sqrt{2}}\|\x\|_{p} \le 1 \right\rbrace$,} & {$\|\x\|_{q}^2 + \epsilon\|\x\|_{2}^2$} \\
{} & {} & {$\cA_2 = \left\lbrace \x \in \R^{N} ~|~ \frac{1}{\sqrt{\epsilon}}\sqrt{\x^{T}\x} \le 1 \right\rbrace$.} & {} \\
\hline
$d$-Low Rank & {} & {$\|\cdot\|_{\cA},~ \cA = \cA_1 + \cA_2,$ where}  & {} \\
{+ $s$-Sparse} & {$2\sqrt{2(16d+1)\ln (s+1)T}$} & {$\cA_1 =  \left\lbrace \x \in \R^{N} ~|~ \sqrt{\x^{\top}{\H}^{-1}\x} \le 1 \right\rbrace$,} & {$\|\x\|_{\H}^2 + \|\x\|_{q}^2$} \\
{} & {} & {$\cA_2 = \left\lbrace \x \in \R^{N} ~|~ \frac{1}{\sqrt{2}}\|\x\|_{p} \le 1 \right\rbrace$.} & {} \\
\hline
\end{tabular}
\endgroup
\caption{Our Results for Additive Loss Spaces}
\label{tab:omd_add_egs}
\end{center}
\end{table}

\begin{restatable} [Noisy Low Rank]{corollary}{nr}
\label{corr:omd_add_nr}
Suppose $\cL_1 = \{\l \in [0,1]^{N} ~|~  \l = 
\U \v ~\}$ is a $d$ rank loss space $(1 \le d \le \ln N)$, perturbed with noisy losses $\cL_2 = \{ \l \in [0,1]^{N} ~|~ \|\l\|_{2}^{2} \le \epsilon,~ \epsilon >0\}$. Then, the regret of the OMD algorithm over the loss space $\cL = \cL_1 + \cL_2$---with regularizer $R(\x) = \x^{\top}\H\x + {\epsilon}\|\x\|_{2}^2$ and learning rate $\eta^* = \sqrt{\frac{2(16d+\epsilon)}{T}}$---is upper bounded as follows
\begin{align*}
\text{Regret}_{T}
	~ \le ~ 
	\sqrt{2(16d+\epsilon)T} 
	\,.
\end{align*}
\end{restatable}

\begin{restatable}[Noisy Sparse]{corollary}{ns}
\label{corr:omd_add_ns}
Suppose $\cL_1 = \{\l \in [0,1]^{N} ~|~ \|\l\|_{0} = s\}$ is an $s$-sparse loss space $(s \in [N])$, perturbed with noisy losses from $\cL_2 = \{ \l \in [0,1]^{N} ~|~ \|\l\|_{2}^{2} \le \epsilon,~ \epsilon >0 \}$. Then, the regret of the OMD algorithm over the loss space $\cL = \cL_1 + \cL_2$---with regularizer $R(\x) = \|\x\|_{q}^2 + \epsilon\|\x\|_{2}^2$ and learning rate $\eta^* = \sqrt{\frac{1+\epsilon}{(2\ln (s+1)-1)T}}$---is upper bounded as follows
\begin{align*}
\text{Regret}_{T}
	~ \le ~ 
	2\sqrt{2(1+\epsilon)\ln (s+1)T} 
	\,.
\end{align*}
\end{restatable}

\begin{restatable}[Low Rank with Sparse]{corollary}{rs}
\label{corr:omd_add_rs}
Suppose $\cL_1 = \{\l \in [0,1]^{N} ~|~  \l = 
\U \v ~\}$ is a $d$ rank loss space $(1 \le d \le \ln N)$, perturbed with $s$-sparse losses $\cL_2 = \{\l \in [0,1]^{N} ~|~ \|\l\|_{0} = s\}, ~s\in [N]$. Then, the regret of the OMD algorithm over the loss space $\cL = \cL_1 + \cL_2$---with regularizer $R(\x) = \x^{\top}\H\x + \|\x\|_{q}^2$ and learning rate
 $\eta^* = \sqrt{\frac{16d+1}{(2 \ln (s+1) -1)T}}$---is upper bounded as follows 
\begin{align*}
\text{Regret}_{T}
	~ \le ~ 
	2\sqrt{2(16d+1)\ln (s+1)T} 
	\,.
\end{align*}
\end{restatable}

\section{Lower Bounds}
\label{sec:lb}

In this section we will derive lower bounds for online learning with experts' advice problem for different structured loss spaces. We first state the lower bound for a general loss space $\cL \subseteq \R^{N}$; see Theorem~\ref{thm:LB_Gen}. The proof of this theorem appears in Appendix~\ref{app:lb} and is based on a lower-bound result of \cite{Shai+09} for online learning of binary hypotheses classes in terms of its \emph{Littlestone's dimension}.


\begin{restatable}[\textbf{Generic Lower Bound}]{theorem}{LBGen}
\label{thm:LB_Gen}
Given parameters $V>0$ and $s >0$ along with any online learning algorithm, there exists a sequence of $V$-dimensional loss vectors $\l_1,\l_2, \ldots ,\l_T \in \{ 0, \pm s \}^{N}$ of sparsity $2^V \leq N$  (i.e., $\text{rank}\left( \left[ \l_1,\l_2, \ldots ,\l_T \right] \right) =V$ and $\|\l_{t}\|_{0} = 2^V$, for all $t \in [T]$) such that 
\begin{align*}
\text{Regret}_{T}
	~ \ge ~
	2s\sqrt{\frac{VT}{8}} 
	\,.
\end{align*}
\end{restatable}

Proof of the above theorem is deferred to Appendix \ref{app:LB_Gen}. The following corollary is a direct consequence of Theorem \ref{thm:LB_Gen}.

\begin{corollary}
\label{cor:lb_vc}
Given parameters $V \in [\ln N]$ and $s>0$ along with any online learning algorithm, there exists a sequence of loss vectors $\l_1,\l_2, \ldots ,\l_T \in [-s, s]^{N}$ of VC-dimension $V$ (i.e., $VC(\{\l_1,\l_2, \ldots ,\l_T\}) = V$), such that
\begin{align*}
\text{Regret}_{T}
	~ \ge ~
	2s\sqrt{\frac{VT}{8}} 
	\,.
\end{align*} 
\end{corollary}

\begin{proof}
Consider the set of loss vectors $L = \{\l \in \{0,\pm s\}^N ~|~ \|\l\|_{0} \le 2^{V} \}$.
From the definition of VC dimension (see Definition \ref{def:vc}, Appendix \ref{app:lb}), it follows that $VC(L) = V$. Hence, Theorem \ref{thm:LB_Gen} implies the stated claim.
\end{proof}

Next we instantiate Theorem \ref{thm:LB_Gen} to derive the regret lower bounds for the structured loss spaces introduced in Section \ref{sec:omd_egs}.
In particular, we begin by stating a lower bound for sparse loss vectors. 
\begin{corollary} {(Lower Bound for Sparse losses)}
\label{corr:lb_sparse}
Given $k \in [N]$ and $s>0$ along with any online learning algorithm, there exists a sequence of loss vectors $\l_1,\l_2, \ldots ,\l_T \in [-s, s]^{N}$ of sparsity $k \in N$ (i.e. $\|\l_t\|_{0} = k $ for all $t \in [T]$) such that

\begin{align*}
\text{Regret}_{T} 
	~ \ge ~
	2s\sqrt{\frac{\floor[\ln k] T}{8}} 
	\,.
\end{align*} 
\end{corollary}
 




\hfill \\

Along the same lines, Theorem~\ref{thm:LB_Gen} leads to a lower bound for losses with small $\ell_p$ norm.

\begin{corollary} {(Lower Bound for $\ell_{p}$ losses)} Given $p \le [\ln N]$ and $s >0$ along with any online learning algorithm, there exists a sequence of loss vectors $\l_1,\l_2, \ldots ,\l_T \in [-s, s]^{N}$ of $\ell_p$ norm at most $s$ (i.e., $\|\l_t\|_{p} \le s$) such that
\begin{align*}
\text{Regret}_{T} 
	~ \ge ~
	s\sqrt{\frac{p T}{8}} 
	\,.
\end{align*}
\end{corollary}
\begin{proof}
Consider the set of all $2^p$-sparse loss vectors in $[-{\frac{s}{2}}, {\frac{s}{2}}]^{N}$. Any such loss vector $\l \in [-{\frac{s}{2}}, {\frac{s}{2}}]^{N}$ has $\|\l\|_p \le s$. The stated claim now follows by applying Theorem \ref{thm:LB_Gen} with parameters ${\frac{s}{2}}$ and $V=p$.
\end{proof}




\begin{corollary} {(Lower Bound for Noisy Losses)}
\label{corr:lb_noisy}
Given $\epsilon>0$ and any online learning algorithm, there exists a sequence of $\epsilon$-noisy loss vectors $\l_1,\l_2, \ldots ,\l_T \in [-\epsilon, \epsilon]^{N}$ (i.e., $\|\l_t\|_{2}^2 \le \epsilon$) such that
\begin{align*}
\text{Regret}_{T} 
	~ \ge ~
	\sqrt{\frac{\epsilon T}{4}} 
	\,.
\end{align*} 
\end{corollary}
 
\begin{proof}
Consider the set of all $2$-sparse loss vectors in $[-\sqrt{\frac{\epsilon}{2}}, \sqrt{\frac{\epsilon}{2}}]^{N}$. Clearly any such loss vector $\l \in [-\sqrt{\frac{\epsilon}{2}}, \sqrt{\frac{\epsilon}{2}}]^{N}$ has $\|\l\|_2^2 \le \epsilon$. 
Hence with parameters $s = \sqrt{\frac{\epsilon}{2}}$ and $V=1$, the result follows directly from theorem \ref{thm:LB_Gen}.
\end{proof}

%
  
\begin{remark}
\label{remark:low-rank}
Note that Theorem~\ref{thm:LB_Gen} (with parameter $V= d, ~s = 1$) recovers the lower bound for low rank loss spaces as established by \cite{HazanEtAl16}: given $1 \le d \le \ln N$ and any online learning algorithm, there exists a sequence of $d$-rank loss vectors $\l_1,\l_2, \ldots ,\l_T \in [-1, 1]^{N}$ such that
\begin{align*}
\text{Regret}_{T}  
	~ \ge ~
	2\sqrt{\frac{d T}{8}} 
	\,.
\end{align*}  
\end{remark}

\hfill \\

We next derive the regret lower bounds for few instances of additive loss spaces. 
\begin{corollary} {(Lower Bound for Noisy Low Rank)}
Given parameters $\epsilon>0$ and $d \in [ \ln N]$ along with any online learning algorithm, there exists a sequence of loss vectors $\l_1,\l_2, \ldots ,\l_T \in [-(1+\epsilon), (1+\epsilon)]^{N}$, where $\l_t = \l_{t1}+\l_{t2}$, with $\l_{t1} \in \{\l \in [-1,1]^{N} ~|~ \l = \U\v\}$ $(\U \in \R^{N\times d}$ is a rank $d$ matrix$)$, and $\|\l_{t2}\|_{2}^2 \le \epsilon$, such that
\begin{align*}
\text{Regret}_{T} 
	~ \ge ~
	2\left( 1+\sqrt{\frac{\epsilon}{2^{d}}} \right)\sqrt{\frac{dT}{8}} 
	\,.
\end{align*} 
\end{corollary}
 
\begin{proof}
Let $N = 2^d$. Consider the matrix $\H \in \{-1,1\}^{N\times d}$ where $2^d$ rows of $\H$ represent $2^d$ vertices of the $d$-hypercube in $[-1,1]^{N}$. Let, 
$
\cL_1 = \{ \H(:,1), \ldots \H(:,d)  \},
$ 
and 
$
\cL_2 = \left\lbrace \l \in \left\lbrace -\sqrt{\frac{\epsilon}{2^{d}}},\sqrt{\frac{\epsilon}{2^{d}}} \right\rbrace ^{N} |~ \|l\|_2^2 = \epsilon \right\rbrace.
$
Note that any loss vectors in $\cL_2$ is $2^d$-sparse.
Consider $\cL = \cL_1 + \cL_2$. The result now follows from Theorem \ref{thm:LB_Gen}, noting that---with $s = \left( 1+\sqrt{\frac{\epsilon}{2^{d}}} \right)$ and $V = d$---the lowering-bounding loss vectors assured in Theorem \ref{thm:LB_Gen}, $\l_1, \ldots ,\l_T$, are contained in $\cL$. 
\end{proof}

\begin{corollary}{(Lower Bound for Noisy Sparse)}
Given parameters $\epsilon>0$ and $k \in [N]$ along with any online learning algorithm, there exists a sequence of loss vectors $\l_1,\l_2, \ldots ,\l_T \in [-(1+\epsilon), (1+\epsilon)]^{N}$, where $\l_t = \l_{t1}+\l_{t2}$, with $\l_{t1} \in \{\l \in [-1,1]^{N} ~|~ \|\l\|_0 \le k\}$, and $\|\l_{t2}\|_{2}^2 = \epsilon$, such that
\begin{align*}
\text{Regret}_{T} 
	~ \ge ~
	2\left( 1+\sqrt{\frac{\epsilon}{k}} \right)\sqrt{\frac{\floor[\ln k]T}{8}} 
	\,.
\end{align*} 
\end{corollary}

\begin{proof}
Consider the following set of loss vectors:
$
\cL_1 = \{ l \in \{-1,1\}^{N} ~|~ \|l\|_0 = k \},
$
and 
$
\cL_2 = \left\lbrace l \in \left\lbrace -\sqrt{\frac{\epsilon}{k}},\sqrt{\frac{\epsilon}{k}} \right\rbrace ^{N} |~ \|l\|_2^2 = \epsilon \right\rbrace.
$
Note that any loss vectors in $\cL_2$ is $k$-sparse. Write $\cL = \cL_1 + \cL_2$. 
The corollary now follows from Theorem \ref{thm:LB_Gen}, noting that---with $s = \left( 1+\sqrt{\frac{\epsilon}{k}} \right)$ and $V = \floor[\ln k]$---the lowering-bounding loss vectors assured in Theorem \ref{thm:LB_Gen}, $\l_1, \ldots ,\l_T$, are contained in $\cL$. 
\end{proof}


%

\section{Conclusion}
In this paper, we have developed a theoretical framework for online learning with structured losses, namely the broad class of problems with additive loss spaces. The framework yields both algorithms that generalize standard online mirror descent and also novel regret upper bounds for relevant settings such as noisy + sparse, noisy + low-rank, and sparse + low-rank losses. In addition, we have derived lower bounds---i.e., fundamental limits---on regret for a variety of online learning problems with structured loss spaces. In light of these results, tightening the gap between the upper and lower bounds for structured loss spaces is a natural, open problem. 
 
Another relevant thread of research is to study settings wherein the learner knows that the loss space is structured, but is oblivious to the exact instantiation of the loss space, e.g., the losses might be perturbations of vectors from a low-dimensional subspace, but, a priori, the learning algorithm might not know the underlying subspace.\footnote{The result of \cite{HazanEtAl16} address the noiseless version of this problem.} Addressing structured loss spaces in bandit settings also remains an interesting direction for future work.

\bibliography{Online-Structured-Learning-arXiv}

\appendix

\section*{\LARGE{Supplementary Material: Online Learning for Structured Losses}}
\vspace*{1cm}

\section{Proof of Theorem \ref{thm:omd}}
\label{app:omd}

\omd*

\begin{proof}  
Consider $\p \in \Delta_N$. We have for all $ t \in[T]$:

\begin{align*}
 \l_{t}.\p_{t} - & \l_{t}.\p 
    \le 
    \l_{t}.(\p_{t}-\p)  
    \,, \quad 
\\
    = &
    \left( \frac{\nabla R(\p_{t}) - \nabla R(\tilde{\p}_{t+1})}{\eta} \right)\cdot(\p_{t} - \p) 
\\
    = &
    \frac{1}{\eta} \Big( B_{R}(\p,\p_{t}) - B_{R}(\p,\tilde{\p}_{t+1}) + B_{R}(\p_{t},\tilde{\p}_{t+1}) \Big)
\\
    \le & 
    \frac{1}{\eta} \Big( 
    	B_{R}(\p,\p_{t}) - B_{R}(\p,{\p}_{t+1}) 
\\	& \quad\quad\quad\quad    	
    	- B_{R}(\p_{t+1},\tilde{\p}_{t+1}) + B_{R}(\p_{t},\tilde{\p}_{t+1}) 
    \Big),
\end{align*}
where the last inequality follows from generalized pythagorean inequality (see Lemma $5.4$, \cite{Bubeck11}). Summing over $t=1\ldots T$, we thus get
\begin{align*}	
\sum_{t=1}^{T} \big( \l_{t}.\p_{t} - \l_{t}.\p \big)    
    & \le 
    \frac{1}{\eta} \Big( B_{R}(\p,\p_{1}) - B_{R}(\p,\p_{T+1}) \Big)
\\	
    + \frac{1}{\eta} \sum_{t=1}^{T} & \Big( B_{R}(\p_{t},\tilde{\p}_{t+1}) - B_{R}\left( \p_{t+1},\tilde{\p}_{t+1} \right) \Big)
    \,.
\end{align*}
Now, it can be shown that
\begin{align*}
B_{R}(\p,\p_{1}) & =  R(\p) - R(\p_{1}) - \nabla R(\p_1).(\p - \p_1) 
\\ & \le  R(\p) - R(\p_{1}) ~ \le ~ D^{2}
	\,.
\end{align*}
The last inequality holds since $\p_1 := \argmin_{\p in  \Delta_N} R(\p)$, which implies $\nabla R(\p_1).(\p - \p_1) > 0$ for all $\p \in \Delta_{N}$. Also, we have
\begin{align*}
& \sum_{t=1}^{T}  \Big( B_{R}(\p_{t},\tilde{\p}_{t+1}) -  B_{R}(\p_{t+1},\tilde{\p}_{t+1})  \Big) 
\\
	& = 
	\sum_{t = 1}^{T} \Big( R(\p_{t}) - R(\p_{t+1}) - \nabla R(\tilde{\p}_{t+1})\cdot(\p_{t} - \p_{t+1}) \Big)
\\
	& \le 
	\sum_{t = 1}^{T}\bigg( 
		\left( \nabla R({\p}_{t})\cdot(\p_{t} - \p_{t+1}) - \frac{\alpha}{2}\|\p_{t} - \p_{t+1}\|^{*2} \right) \\
	& - \nabla R(\tilde{\p}_{t+1}).(\p_{t} - \p_{t+1}) 
	\bigg),
	\text{ (by strong convexity of $R$)}
	\\ 
	& = 
	\sum_{t = 1}^{T} \Big( 
		-\eta \l_t(\p_{t})\cdot(\p_{t} -\p_{t+1}) -\frac{\alpha}{2}\|\p_{t}-\p_{t+1}\|^{*2} 
	\Big) 
\\
	& \le 
	\sum_{t = 1}^{T} \Big( \eta G \|\p_{t}-\p_{t+1}\|^{*} -\frac{\alpha}{2}\|\p_{t}-\p_{t+1}\|^{*2} \Big), 
\\	& \quad\quad\quad\quad\quad\quad\quad 
      \quad\quad\quad\quad\quad\quad\quad \text{(by H\"olders inequality)}
\\
	& \le
	\sum_{t=1}^{T}\frac{\eta^{2}G^{2}}{2\alpha}
	\,, \text{ (as $\textstyle{\frac{\eta^{2}G^{2}}{2\alpha} + \frac{\alpha}{2}\|\p_{t}-\p_{t+1}\|^{*2} -\eta G \|\p_{t}-\p_{t+1}\|^{*}}$}
\\& \quad\quad\quad\quad\quad\quad\quad\quad
	\text{$\textstyle{ = \left( \frac{\eta G}{\sqrt{2\alpha}} - \sqrt{\frac{\alpha}{2}}\|\p_{t}-\p_{t+1}\|^{*} \right)^{2} \ge 0}$)}
\\
	& =
	\frac{\eta^{2}G^{2}T}{2\alpha} 
	\,.
\end{align*}
Thus, we get 
$
\sum_{t=1}^{T} \big( \l_{t}.\p_{t} - \l_{t}.\p \big) 
	\leq 
	\frac{1}{\eta}\left( D^{2} + \frac{\eta^{2}G^{2}T}{2\alpha} \right).
$
Note that above bound holds for any $\p \in \Delta_{N}$. Therefore, 
\begin{align*}
\sum_{t=1}^{T} \big( \l_{t}.\p_{t} - \inf_{\p \in \Delta_{N}}\l_{t}.\p \big) 
	& = 
	\sum_{t=1}^{T} \big( \l_{t}.\p_{t} - \min_{i = 1}^{N}l_{t}(i) \big)\\ 
	& \leq 
	\frac{1}{\eta}\left( D^{2} + \frac{\eta^{2}G^{2}T}{2\alpha} \right) 
	\,.
\end{align*}
Where the first equality holds since $  \argmin_{p \in \Delta_N} \l.\p  \in \{\e_1,\e_2,\ldots,\e_N\}$.
Note that the right-hand-side is minimized at $\eta^{*} = \frac{D}{G}\sqrt{\frac{2\alpha}{T}}$; substituting this back in the above inequality gives the desired result.
\end{proof}

\section{Proofs from Section \ref{sec:omd_struct}}
\label{app:omd_prfs}

\begin{enumerate}
\item {Sparse loss space:} $\cL = \{\l \in [0,1]^{N} ~|~ \|\l\|_{0} = s\}$, $s$ being the loss sparsity, $1 \le s \le N$. Then using $q$-norm, $R(\x) = \|\x \|_{ q}^2 = \left(\sum_{i = 1}^{N}(x_i^{q})\right)^{\frac{2}{q}}$, where $q = \frac{\ln s'}{\ln s' - 1}$, $s' = (s+1)^2$, as the regularizer, we get,
\begin{align*}
\text{Regret}_{T}  
	~ \le ~ 
	2\sqrt{{\ln (s+1)T}} 
	\,.
\end{align*}

\begin{proof}
Note that $4 \le (s+1)^2 \le (N+1)^2$. Let $p = \ln s' = 2 \ln (s+1)$. Clearly, $2 \le p \le 2 \ln (N+1)$.
Consider the norm $\|\cdot\| = \frac{1}{\sqrt{2}}\|\cdot\|_{p}$, and it dual $\|\cdot\|^{*} = {\sqrt{2}}\|\cdot\|_{q}$, where $\frac{1}{q} = 1 - \frac{1}{p}$, or $q = \frac{\ln s'}{\ln s' - 1} \in (1,2]$. Note that: 
\begin{enumerate}
\item For $\l \in \cL, ~\|\l\| \le 1$ (since $\l$ is at most $s$-sparse). Hence $G = 1$.
\item For any $\p \in \Delta_{N}$,  $ R(\p) - R(\p_1) \le R(\p) \le 1$. Hence $D = 1$.
\item $R(\x) = \|\x\|_{q}^2$ is $(q-1)$-strongly convex with respect to $\sqrt{2} \|\cdot\|_{q}$, for all $ \x \in \Delta_{N}$ (see Lemma $17$ of Appendix A in \cite{ShalevShwartz07}). Hence $\alpha = (q-1) = \frac{1}{(2\ln (s+1) - 1)}$. 
\end{enumerate}
The result now follows from an application of Theorem \ref{thm:omd}.
\end{proof}

\item {Spherical loss space:} $\cL = \{\l \in [0,1]^{N} ~|~ \|\l\|_{\A} = \l^{\top}\A\l \le \epsilon\}$, where $\A$ is a positive definite matrix, $\epsilon > 0$. Then using the ellipsoidal norm $R(\x) = \epsilon\x^{\top}\A^{-1}\x$, as the regularizer, we get,
\begin{align*}
\text{Regret}_{T}
	~ \le ~ 
	\sqrt{\lambda_{\max}(\A^{-1}){\epsilon T}} 
	\,,
\end{align*}
where $\lambda_{\max}(\A^{-1})$ denotes the maximum eigenvalue of $\A^{-1}$.

\begin{proof}
Consider the norm $\|\cdot\| = \frac{1}{\sqrt{\epsilon}}\|\cdot\|_{\A}$, where for any $\x \in \R^N$, $\|\x\|_{\A} = \sqrt{\x^{\top}\A\x}$. Note that the dual norm is $\|\cdot\|^{*} = \sqrt{\epsilon}\|\cdot\|_{\A^{-1}}$. In addition, we have 
\begin{enumerate}
\item For any $\l \in \cL, ~\|\l\| \le 1$. Hence, $G = 1$.
\item For any $\p \in \Delta_{N}$, $R(\p) - R(\p_1) \leq \epsilon \p^{\top} \A^{-1} \p \leq \epsilon \lambda_{\max}(\A^{-1}) $; here, the first inequality follows from the fact that $\A^{-1}$ is positive definite. Hence $D = \sqrt{\epsilon \ \lambda_{\max}(\A^{-1})}$.
\item $R(\x) = \epsilon\x^{\top}\A^{-1}\x$ is $2$-strongly convex with respect to $\sqrt{\epsilon}\|\cdot\|_{\A^{-1}}$, for all $ \x \in \Delta_{N}$. Hence $\alpha = 2$. 
\end{enumerate}
As before, an application of Theorem \ref{thm:omd} gives us the result.
\end{proof}

\item {Noisy loss:} $\cL = \{\l \in [0,1]^{N} ~|~ \|\l\|_{2}^{2} \le \epsilon\}, ~\epsilon > 0$. Then using the standard euclidean norm $R(\x) = \epsilon\|\x\|_{2}^2$, as the regularizer, we get,
\begin{align*}
\text{Regret}_{T}
	~ \le ~ 
	\sqrt{{\epsilon T}} 
	\,.
\end{align*}

\begin{proof}
Consider the norm $\|\cdot\| = \frac{1}{\sqrt{\epsilon}}\|\cdot\|_{2}$, and it dual $\|\cdot\|^{*} = {\sqrt{\epsilon}}\|\cdot\|_{2}$. Note that: 
\begin{enumerate}
\item For any $\l \in \cL, ~\|\l\| \le 1$. Hence $G = 1$.
\item For any $\p \in \Delta_{N}$, $R(\p) - R(\p_1) = \epsilon \| \p \|_2^2 - \epsilon \| \p_1 \|_2^2 \le \epsilon$. Hence $D = \sqrt{\epsilon}$.
\item $R(\x) = \epsilon\|\x\|_{2}^2$ is $2$-strongly convex with respect to $\sqrt{\epsilon}\|\cdot\|_{2}$, $\forall \x \in \Delta_{N}$. Hence $\alpha = 2$. 
\end{enumerate}
As before the result now follows applying Theorem \ref{thm:omd}. One can also recover this regret bound of \emph{noisy loss} as a special case of \emph{spherical loss} with $\A = \A^{-1} = \I_{N}$, since $\lambda_{\max}(\I_N) = 1$.
\end{proof}

\item {Standard loss space:} $\cL = [0,1]^{n}$. Then using unnormalized negative entropy, $R(\x) = \sum_{i = 1}^{N}x_i \log x_i - \sum_{i = 1}^{N} x_i$, as the regularizer, we get
\begin{align*}
\text{Regret}_{T}  
	~ \le ~ 
	\sqrt{{2T \ln N}} 
	\,.
\end{align*}

\begin{proof}
Consider the norm $\|\cdot\| = \|\cdot\|_{\infty}$, and its dual norm $\|\cdot\|^{*} = \|\cdot\|_{1}$. Note that: 
\begin{enumerate}
\item For any $\l \in \cL, ~\|\l\| \le 1$. Hence $G = 1$.
\item For any $\p \in \Delta_{N}$, $R(\p) - R(\p_1) = \sum_{i = 1}^{N}p_i \ln \left( \frac{p_i}{p_{1i}} \right) - \sum_{i = 1}^{N}(p_i - p_{1i}) \le \ln N $ (since $\p_1 = \frac{\1}{N}$, assuming $0\ln 0 = 0$). Hence $D = \sqrt{\ln N}$.
\item $R(\x) = \sum_{i = 1}^{N}x_i \log x_i - \sum_{i = 1}^{N} x_i$ is $1$-strongly convex with respect to $\|\cdot\|_{1}$, $\forall \x \in \Delta_{N}$ (see example $2.5$, \cite{OCO-SS12}). Hence $\alpha = 1$. 
\end{enumerate}
The result now follows via Theorem \ref{thm:omd}.
\end{proof}



\end{enumerate}

\section{Proofs from Section \ref{sec:omd_add_egs}}
\label{app:omd_add_egs}
The proofs given in this section are based on the results given in Section \ref{app:omd_prfs}, which establish regret guarantees of the OMD algorithm for specific  structured loss spaces. \\

\noindent {\textbf{Proof of Corollary \ref{corr:omd_add_nr}}}

\nr* 

\begin{proof}
Consider the following two convex, compact, bounded and centrally symmetric sets
\begin{align*}
& \cA_1 =  \left\lbrace \x \in \R^{N} ~|~ \sqrt{\x^{\top}{\H}^{-1}\x} \le 1 \right\rbrace, \text{ and }
\\
& \cA_2 = \left\lbrace \x \in \R^{N} ~|~ \frac{1}{\sqrt{\epsilon}}\sqrt{\x^{T}\x} \le 1 \right\rbrace,
\end{align*}
where $\H = \I_{N} + \U^{\top} \M \U $, $\M$ being the matrix corresponding to the L\"owner-John ellipsoid of $\cL$ \cite{HazanEtAl16}. We have $\|\x\|_{\cA_1} = \sqrt{\x^{\top}{\H}^{-1}\x}$, and $\|\x\|_{\cA_2} = \frac{1}{\sqrt{\epsilon}}\sqrt{\x^{T}\x}$, for any $\x \in \R^{N}$. Clearly, $\cL_1 \subseteq \cA_1$, and $\cL_2 \subseteq \cA_2$. Consider the norm $\|\cdot\| = \|\cdot\|_{\cA}$, and its dual norm $\|\cdot\|^{*} = \|\cdot\|_{\cA}^{*}$, where $\cA = \cA_1 + \cA_2$. 
Note that, for any $\l \in \cL, ~\|\l\|_{\cA} \le 1$, since $\cL \subseteq \cA$.
Let us choose $R_1(\x) = \x^{\top}\H\x$, and $R_{2}(\x) = \epsilon\|\x\|_{2}^2$.
Recall from Appendix \ref{app:omd_prfs}, 
\begin{enumerate}
\item $R_1(\p) - R_1(\p_1) \le D_1^2 = 16d$, and $R_2(\p) - R_2(\p_1) \le D_2^2 = \epsilon$. Hence $D = \sqrt{16d + \epsilon}$.
\item Both $R_1$ and $R_2$ are $2$-strongly convex w.r.t. $\|\x\|_{\cA_1}^{*} = \sqrt{\x^{\top}{\H}\x}$, and $\|\x\|_{\cA_2}^{*} = \sqrt{\epsilon\x^{T}\x}$ respectively. Hence $\alpha_1 = \alpha_2 = 2$, and $\alpha = \frac{\min\{\alpha_1, \alpha_2\}}{2} = 1$. 
\end{enumerate}
The result now follows applying Theorem \ref{thm:omd_add}.
\end{proof}

\noindent {\textbf{Proof of Corollary \ref{corr:omd_add_ns}}}

\ns* 

\begin{proof}
Let $s' = (s+1)^2$, $p = \ln s' = 2 \ln (s+1)$. Note that $2 \le p \le 2 \ln (N+1)$.
Consider the following two convex, compact, and centrally symmetric sets
\begin{align*}
& \cA_1 = \left\lbrace \x \in \R^{N} ~|~ \frac{1}{\sqrt{2}}\|\x\|_{p} \le 1 \right\rbrace, \text{ and }
\\
& \cA_2 = \left\lbrace \x \in \R^{N} ~|~ \frac{1}{\sqrt{\epsilon}}\sqrt{\x^{T}\x} \le 1 \right\rbrace.
\end{align*}
 We have $\|\x\|_{\cA_1} = \frac{1}{\sqrt{2}}\|\x\|_{p}$, and $\|\x\|_{\cA_2} = \frac{1}{\sqrt{\epsilon}}\sqrt{\x^{T}\x}$, for any $\x \in \R^{N}$. We have $\cL_1 \subseteq \cA_1$ and $\cL_2 \subseteq \cA_2$. Consider the norm $\|\cdot\| = \|\cdot\|_{\cA}$, and its dual norm $\|\cdot\|^{*} = \|\cdot\|_{\cA}^{*}$, where $\cA = \cA_1 + \cA_2$. 
Note that, for any $\l \in \cL, ~\|\l\|_{\cA} \le 1$, since $\cL \subseteq \cA$.
Let us choose $R_1(\x) = \|\x\|_{q}^2$, where $\frac{1}{q} = 1 - \frac{1}{p}$, or $q = \frac{\ln s'}{\ln s' - 1} \in (1,2]$, and $R_{2}(\x) = \epsilon\|\x\|_{2}^2$.
Recall from Appendix \ref{app:omd_prfs}, 
\begin{enumerate}
\item $R_1(\p) - R_1(\p_1) \le D_1^2 = 1$, and $R_2(\p) - R_2(\p_1) \le D_2^2 = \epsilon$. Hence $D = \sqrt{1 + \epsilon}$.
\item $R_1$ is $(q-1)$-strongly convex w.r.t. $\|\x\|_{\cA_1}^{*} = {\sqrt{2}}\|\x\|_{q}$, and $R_2$ is $2$-strongly convex w.r.t. $\|\x\|_{\cA_2}^{*} = \sqrt{\epsilon\x^{T}\x}$. Hence $\alpha_1 = (q-1)$, $\alpha_2 = 2$, and $\alpha = \frac{\min\{\alpha_1, \alpha_2\}}{2} = (q-1)$, since $(q-1) \in (0,1]$. 
\end{enumerate}
Using Theorem \ref{thm:omd_add}, we get the desired claim. 
\end{proof}

\noindent {\textbf{Proof of Corollary \ref{corr:omd_add_rs}}}

\rs* 

\begin{proof}
Let $s' = (s+1)^2$, $p = \ln s' = 2 \ln (s+1)$. Clearly, $2 \le p \le 2 \ln (N+1)$. Also let $\H = \I_{N} + \U^{\top} \M \U $, $\M$ being the matrix corresponding to the L\"owner-John ellipsoid of $\cL$; this regularizer was used in \cite{HazanEtAl16}. Now consider the following two convex, compact, and centrally symmetric sets
\begin{align*}
& \cA_1 = \left\lbrace \x \in \R^{N} ~|~ \sqrt{\x^{\top}{\H}^{-1}\x} \le 1  \right\rbrace \text{ and }
\\
& \cA_2 = \left\lbrace \x \in \R^{N} ~|~ \frac{1}{\sqrt{2}}\|\x\|_{p} \le 1 \right\rbrace.
\end{align*}
 We have $\|\x\|_{\cA_1} = \sqrt{\x^{\top}{\H}^{-1}\x}$ and $\|\x\|_{\cA_2} = \frac{1}{\sqrt{2}}\|\x\|_{p}$, for any $\x \in \R^{N}$. Clearly, $\cL_1 \subseteq \cA_1$, and $\cL_2 \subseteq \cA_2$. Consider the norm $\|\cdot\| = \|\cdot\|_{\cA}$, and its dual norm $\|\cdot\|^{*} = \|\cdot\|_{\cA}^{*}$, where $\cA = \cA_1 + \cA_2$. 
Note that, for any $\l \in \cL, ~\|\l\|_{\cA} \le 1$, since $\cL \subseteq \cA$.
Let us choose $R_1(\x) = \x^{\top}\H\x$, and $R_{2}(\x) = \|\x\|_{q}^2$, where $\frac{1}{q} = 1 - \frac{1}{p}$, or $q = \frac{\ln s'}{\ln s' - 1} \in (1,2]$.
Recall from Appendix \ref{app:omd_prfs}, 
\begin{enumerate}
\item $R_1(\p) - R_1(\p_1) \le D_1^2 = 16d$, and $R_2(\p) - R_2(\p_1) \le D_2^2 = 2\epsilon$. Hence $D = \sqrt{16d + 2\epsilon}$.
\item $R_1$ is $2$-strongly convex w.r.t. $\|\x\|_{\cA_1}^{*} = \sqrt{\x^{\top}{\H}\x}$, and $R_2$ is $(q-1)$-strongly convex w.r.t. $\|\x\|_{\cA_1}^{*} = {\sqrt{2}}\|\x\|_{q}$. Hence $\alpha_1 = 2$, $\alpha_2 = (q-1)$, and $\alpha = \frac{\min\{\alpha_1, \alpha_2\}}{2} = (q-1)$, since $(q-1) \in (0,1]$. 
\end{enumerate}
As in the previous corollaries, the claim follows by applying Theorem \ref{thm:omd_add}.
\end{proof}

\section{Proofs from Section \ref{sec:lb}}
\label{app:lb}

This section provides a simple generalization of a lower-bound result of \cite{Shai+09} for online learning of \emph{binary hypotheses classes}. Then, using this generalization, it establishes Theorem~\ref{thm:LB_Gen}. 

We begin by defining the binary hypothesis learning problem and Littlestone's dimension of a set of binary hypotheses. 
\begin{definition}\textbf{Online Binary Hypothesis Learning Problem:}
\label{def:obhl}
For a given instance space $\X$, binary label space  $\Y = \{0,1\}$, and a class of binary hypothesis functions $\cH = \{h_1, \ldots, h_n\}$, $h_i: \X \mapsto \{0,1\}$ $\forall i \in [n]$, the problem of online binary hypothesis learning is a sequential prediction game between the environment and a learner. At each iteration, the environment provides an instance $x_t \in \X$, and the learner's objective is to predict its class label $\hat{y}_t \in \hat{\Y} = \{0,1\}$. At the end of $T$ iterations, the performance of the learner is measured in terms of its number of mispredictions with respect to the best hypothesis in the $\cH$,  termed as the regret of the learner, defined as follows:
\begin{align*}
\text{Regret}_{T} = \sum_{t = 1}^{T}|\hat{y}_t - y_t| - \min_{h \in \cH} \sum_{t=1}^T |h(x_t) - y_t|
\end{align*} 
\end{definition}

\begin{definition}\textbf{Littlestone's-dimension of a set of binary hypotheses (\cite{Littlestone88,Shai+09}):}
\label{def:ldim}
The concept was first introduced by Littlestone \cite{Littlestone88} as a measure of complexity of hypothesis classes that are learnable in an online setting. More specifically, for an online binary hypothesis learning problem, let $\cH$ be a non-empty class of binary hypotheses such that $h:\X \mapsto \{0,1\}$ for all $ h \in \cH$, where $\X$ is the instance space. An instance-labeled tree is said to be shattered by the class $\cH$ if for any root-to-leaf path $(x_1 , y_1 ), \ldots , (x_d , y_d )$ there is some $h \in \cH$ such that for all $i \le d$, $h(x_i ) = y_i$ for all $i$. The Littlestone's dimension of $\cH$, $Ldim(\cH)$ is the largest integer $d$ such that there exist a full binary tree of depth $d$ (i.e., any of its branch contains $d$-many non-leaf nodes) that is shattered by $\cH$.
\end{definition}

\begin{definition}\textbf{VC-dimension of a set (\cite{Alon14}):}
\label{def:vc}
Let $A = \{\a_1, \a_2, \ldots, \a_m\}$ be a set of $m$,  $d$-dimensional vectors, $\a_i \in \R^d, ~\forall i \in [m]$. Let $C = \{c_1, \ldots, c_k \} \subseteq [m]$ be a subset of vectors of set $A$. We say that $A$ shatters $C$ with respect to the real numbers $(t_{c_1}, \ldots, t_{c_k} )$, iff for any $D \subseteq C$, there is a coordinate $i \in [d]$ with $A(i, c) < t_c$ for all $c \in D$, and $A(i, c) > t_c$ for all $c \in C \setminus D$. Then the VC dimension of the set of vectors $A$, denoted by $VC(A)$, is defined as the maximal size subset of vectors shattered by $A$. Clearly, $VC(A) \le \ln d$.
\end{definition}

Let us first recall the lower-bound result of \cite{Shai+09} for online learning of binary hypotheses classes in terms of its {Littlestone's dimension}.

\begin{lemma} \textbf{\cite{Shai+09}}
\label{lem:LB_Shai}
 Let $\X$ and $\Y = \{0,1\}$ respectively denote the instance and label space for an online binary hypothesis learning problem, and $\cH$ be a class of binary hypotheses such that $h: \X \mapsto \{0,1\}$ for all $ h \in \cH$. 
Then for any (possibly randomized) algorithm for the classification problem, there exists a sequence of labeled instances $(x_1,y_1), \ldots, (x_T,y_T) \in (\X \times \Y)$ such that
\begin{align*}
\E\left[ \sum_{t = 1}^{T}|\hat{y}_t - y_t| \right] - \min_{h \in \cH} \sum_{t = 1}^{T}|h({x}_t) - y_t|  
	 \ge 
	\sqrt{\frac{\text{Ldim}(\cH)T}{8}},
\end{align*}
where $\hat{y}_t \in \{0,1\}$ is the algorithm's output at iteration $t$.
\end{lemma}

Now suppose in the problem of binary hypothesis learning, the learner is allowed to make predictions in $\left \lbrace 0, \frac{1}{2}, 1 \right \rbrace$, i.e. $\hat{\Y} = \left \lbrace 0, \frac{1}{2}, 1 \right \rbrace$. We show that the lower bound guarantee of Lemma \ref{lem:LB_Shai} holds as it is even for this problem. Formally,

\begin{lemma}
\label{lem:LB_Shai_mod} 
Let $\X$ and $\Y = \{0,1\}$ respectively denote the instance and label space for an online binary hypothesis learning problem, and $\cH$ be a class of binary hypotheses such that $h: \X \mapsto \left \lbrace 0, \frac{1}{2}, 1 \right \rbrace$ for all $h \in \cH$. 
Then for any (possibly randomized) algorithm for the problem, there exists a sequence of labeled instances $(x_1,y_1), \ldots, (x_T,y_T) \in (\X \times \Y)$, such that
\begin{align*}
\E\left[ \sum_{t = 1}^{T}|\hat{y}_t - y_t| \right] - \min_{h \in \cH} \sum_{t = 1}^{T}|h({x}_t) - y_t|  
	\ge 
	\sqrt{\frac{\text{Ldim}(\cH)T}{8}}
	,
\end{align*}
where $\hat{y}_t \in \left \lbrace 0, \frac{1}{2}, 1 \right \rbrace$ is the algorithm's output at iteration $t$.
\end{lemma}

\begin{proof}
Let $d = Ldim(\cH)$ and, for simplicity, assume that $T$ is an integer multiple of $d$, say, $T = kd$ for some non-negative integer $k$. Consider a full binary $\cH$-shattered tree of
depth $d$. We construct the sequence $(x_1 , y_1 ), (x_2 , y_2 ), \ldots , (x_T , y_T )$ by following a root-to-leaf path $(u_1 , z_1 ), (u_2 , z_2 ),\ldots , (u_d , z_d )$ in the shattered tree. We pick the path in a top-down fashion starting at the root. The label $z_i \in \{0, 1\}$
determines whether the path moves to the left or to the right subtree of $u_i$ and it thus determines $u_{i+1}$. Each node $u_i$ on the path, $i \in [d]$, corresponds
to a block $(x_{(i-1)k+1}, y_{(i-1)k+1}), \ldots , (x_{ik}, y_{ik})$ of $k$ examples. We define 
$x_{(i-1)k+1} = x_{(i-1)k+2} = \cdots = x_{ik} = u_i$, and we choose $y_{(i-1)k+1}, \ldots, y_{ik} \in \{0,1\}$, independently uniformly at random. For each block, let $T_i = \{(i - 1)k + 1, . . . , ik\}$
be the time indices of the $i^{th}$ block. Denote $r = \sum_{t \in T_i}y_t$. Note that since $y_t \in \{0,1\}$,
\begin{align*}
\min_{z_i \in \left \lbrace 0, \frac{1}{2}, 1 \right \rbrace}\sum_{t \in T_i}|z_i - y_t| & = \min_{z_i \in \left \lbrace0,1\right \rbrace}\sum_{t \in T_i}|z_i - y_t| 
\\ & = 
\begin{cases}
    k-r,& \text{ if } r \ge \frac{k}{2}\\
    r,& \text{otherwise.}
\end{cases}
\end{align*}
Note that the expected loss incurred by the learner in the $i^{th}$ block is $k/2$. Hence, $k/2 - \min_{z_i \in \{0,1\}}\sum_{t \in T_i}|z_i - y_t| = |r - k/2|$. Taking expectations over the $y$'s and using Khinchine's inequality \cite{PLG06}, we obtain
\begin{align*}
k/2 - \E \left[ \min_{z_i \in \{0,1\}}\sum_{t \in T_i}|z_i - y_t| \right] = \E[|r - k/2|] \ge \sqrt{\frac{k}{8}}.
\end{align*}
Since $Ldim(\cH) = d$, note that there exists $h \in \cH$, such that within each block we have $h(u_i) = z_i$. Thus by summing over the $d$ blocks we get
\begin{align*}
dk/2 - \E \left[ \min_{h \in \cH}\sum_{t \in T}|h(x_t) - y_t| \right] \ge d\sqrt{\frac{k}{8}}. 
\end{align*}
Finally, since $\frac{dk}{2} = \frac{T}{2} = \E[ \sum_{t = 1}^{T} |\hat{y}_ t - y_t|]$, we conclude that the expected regret, w.r.t. the randomness of choosing the labels, is at least $d\sqrt{\frac{k}{8}} = \sqrt{\frac{dT}{8}}$. Therefore, there exists a particular sequence for which the regret is at least $\sqrt{\frac{dT}{8}}$, which concludes the proof.
\end{proof}

Note that the following corollary follows directly from Lemma \ref{lem:LB_Shai_mod}:

\begin{corollary}
\label{corr:LB_Shai_scaled} 
Consider any $s > 0$. Let $\X$ and $\Y = \{0,s\}$ respectively denote the instance and label space for an online binary hypothesis learning problem, and $\cH$ be a class of binary hypotheses such that $h: \X \mapsto \left \lbrace 0, \frac{s}{2}, s \right \rbrace$  for all $ h \in \cH$. 
Then for any (possibly randomized) algorithm for the problem, there exists a sequence of labeled instances $(x_1,y_1), \ldots, (x_T,y_T)$, $(x_t,y_t) \in (\X \times \Y)$, for $t \in [T]$, such that
\begin{align*}
\E\left[ \sum_{t = 1}^{T}|\hat{y}_t - y_t| \right] - \min_{h \in \cH} \sum_{t = 1}^{T}|h({x}_t) - y_t|  
	\ge 
	s\sqrt{\frac{\text{Ldim}(\cH)T}{8}},
\end{align*}
where $\hat{y}_t \in \left \lbrace 0, \frac{s}{2}, s \right \rbrace$ is the algorithm's output at iteration $t$.
\end{corollary}

\subsection{Proof of Theorem \ref{thm:LB_Gen}}
\label{app:LB_Gen}

\LBGen*

\begin{proof}
We first construct a problem instance of an online binary hypothesis learning problem (see Definition \ref{def:obhl}) as follows: Let $\X = \{\e_1, \e_2, \ldots, \e_V \}$ and $\Y = \{0, s\}$ respectively denote the instance and label space. Consider the $2^V$ vertices of the $V$-dimensional binary hypercube $\u_1,\u_2, \ldots \u_{2^V}$, where $\u_i \in \{\pm 1\}^{V}$ represents the $i^{th}$ vertex of the hypercube. We define a matrix $\U \in \{\pm1, 0\}^{N \times V}$ such that
\begin{align*}
    U(i,j) = 
\begin{cases}
    u_{i}(j),& \forall i \in [2^V],~ j \in [V]\\
    0,& \text{otherwise}
\end{cases}
\end{align*}
That is, the first $2^V$ rows of $\U$ are the $\u_i$s and the remaining $(N - 2^V)$ rows are all zeros. Now, let us consider the following class of hypothesis functions $\H = \{h_1, h_2, \ldots h_N\}$ $h_{i}: \X \mapsto \left \lbrace 0, \frac{s}{2}, s \right \rbrace ~\forall i \in [N]$, such that
\begin{align*}
    h_{i}(\e_j) = 
\begin{cases}
    0,& \text{if } U(i,j) = -1\\
    \frac{s}{2},& \text{if } U(i,j) = 0\\
    s,& \text{if } U(i,j) = +1
\end{cases}
\end{align*}
Note that $Ldim(\cH) = V$. Thus using Corollary \ref{corr:LB_Shai_scaled} we get 
\begin{align*}
\E_{h_i \sim \p_t}\left[ \sum_{t = 1}^{T}|h_i(\v_t) - y_t| \right] & - \min_{h \in \cH} \sum_{t = 1}^{T}|h(\v_t) - y_t| 
\\ 
	& \ge 
	s\sqrt{\frac{VT}{8}} 
	\,,
\end{align*}
where $\v_t \in \X$ is the observed instance, $y_t \in \Y$ is the true label of $\v_t$, and $\p_t \in \Delta_{N}$ is the distribution maintained by the algorithm over the set of $N$ classifiers $\cH$, at iteration $t$. Let us denote algorithms expected prediction at time $t$ as $\hat{y}_{t} = \sum_{i = 1}^{N}p_t(i)h_i(\v_t)$.
Note that for any $y \in \{0,s\}$, and $\v \in \X$, $|h_i(\v) - y|  = \frac{s-(2y-s).\U(i,:).\v}{2}$. Now let us construct the sequence of loss vectors $\l_t = \bar{y}_t\U\v_t$, where $\bar{y}_t = (s-2y_t) \in \{-s,s\}$. Note that $\l_{t} \in \left \lbrace 0, \frac{s}{2}, s \right \rbrace$. Since the columns of matrix $\U$ is $2^V$ sparse and each instance vector $\v_t \in \{\e_1, \ldots, \e_V \}$, we have $\|\l_t\|_0 = 2^V$, for all $t \in [T]$. In addition, since the rank of matrix $\U$ is $V$, the dimensionality constraint on the loss vectors is satisfied as well: $\text{rank}\left( \left[ \l_1,\l_2, \ldots ,\l_T \right] \right) =V$. Note that the regret incurred by the learner in this case is
\begin{align*}
\E_{h_i \sim \p_t} & \left[ \sum_{t = 1}^{T} 
\frac{s - (2y_t -s)\U(i,:)\v_t}{2} \right] \\
& - \min_{h_i \in \cH} \sum_{t = 1}^{T}\frac{s-(2y_t-s)\U(i,:)\v_t}{2}
	~ \ge ~
	s\sqrt{\frac{VT}{8}} 
	\,.
\end{align*}
Equivalently,
\begin{align*}
\sum_{t = 1}^{T} \sum_{i=1}^{N}\p_{t}(i){\bar{y}_t\U(i,:)\v_t} - \min_{i \in N} \sum_{t = 1}^{T}{\bar{y}_t\U(i,:)\v_t}
	\ge
	2s\sqrt{\frac{VT}{8}}.
\end{align*}
Since $\l_t = \bar{y}_t\U\v_t$, this further gives
\begin{align*}
\sum_{t = 1}^{T} \p_{t}\l_t - \min_{i \in N} \sum_{t = 1}^{T}\l_t(i)
	~ \ge ~
	2s\sqrt{\frac{VT}{8}} 
	\,.
\end{align*}
Hence the desired result follows.
\end{proof}


\section{Tight Examples for Theorem~\ref{thm:omd_add}}
\label{appendix:tight-examples}

In this section we provide examples in which the regret guarantee of Theorem~\ref{thm:omd_add} is essentially tight. That is, we present loss spaces $\cL_1$ and $\cL_2$ such that OMD algorithm obtained via Theorem~\ref{thm:omd_add} provides an order-wise optimal regret guarantee for the additive loss space $\cL = \cL_1 + \cL_2$. \\

\noindent
\textbf{Composition of Low Ranks:} 
Let $\cL_1 = \{\l \in [0,1]^{N} ~|~  \l = 
\U_1 \v ~\}$ and $\cL_2 = \{ \l \in [0,1]^{N} ~|~  \l = 
\U_2 \v \}$ be loss spaces of rank $d_1$ and $d_2$, respectively (i.e., rank of the matrices $\U_1$ and $\U_2$ are respectively $d_1$ and $d_2$). Here $(d_1+d_2) \le \ln N$. Consider the regularizer $R(\x) = \x^{\top}(\H_1 + \H_2)\x$, where $\H_1 = \I_N + \U_{1}^{\top}\M_1\U_1$, and $\H_2 = \I_N + \U_{2}^{\top}\M_2\U_2$, $\M_1$ and $\M_2$ being the L\"owner John ellipsoid matrix for $\cL_1$ and $\cL_2$. That is, $R(\x) = R_1( \x) + R_2(\x)$, where $R_1( \x) $ and  $R_2(\x)$ are the regularizers for $\cL_1$ and $\cL_2$ respectively. 

Theorem~\ref{thm:omd_add} assets that the OMD algorithm, with regularizer $R$, for the loss space $\cL = \cL_1 + \cL_2$ achieves the following regret bound:
\begin{align*}
\text{Regret}_{T}
	~ \le ~ 
	4\sqrt{2(d_1+d_2)T} 
	\,.
\end{align*}
This regret guarantee is tight, since $Rank(\cL)$ can be as high as  $(d_1 + d_2)$ and, hence, we get a nearly matching lower bound by applying the result of \cite{HazanEtAl16}; see also Remark~\ref{remark:low-rank} in Section~\ref{sec:lb}. \\

\noindent
\textbf{Composition of Noise}
Let  loss spaces $\cL_1 = \{ \l \in [0,1]^{N} ~|~ \|\l\|_{2}^{2} \le \epsilon_1 \}$ and $\cL_2 = \{ \l \in [0,1]^{N} ~|~ \|\l\|_{2}^{2} \le \epsilon_2 \}$. 
Then, via an instantiation of Theorem~\ref{thm:omd_add}, we get that the regret of the OMD algorithm over the loss space $\cL = \cL_1 + \cL_2$, with regularizer $R(\x) = (\epsilon_1+\epsilon_2)\|\x\|_{2}^{2}$ (and $\eta^* = \sqrt{\frac{2(\epsilon_1+\epsilon_2)}{T}}$) is upper bounded as follows:
\begin{align*}
\text{Regret}_{T}
	~ \le ~ 
	\sqrt{2(\epsilon_1+\epsilon_2)T} 
	\,.
\end{align*}
Again, modulo constants, this is the best possible regret guarantee for $\cL$; see Corollary~\ref{corr:lb_noisy}.

\end{document}